\documentclass{kais}

\usepackage[dcucite]{harvard}

\usepackage[T1]{fontenc}
\usepackage{graphicx}
\usepackage{amsmath}
\usepackage{amssymb}
\usepackage{booktabs}
\usepackage{lmodern}
\usepackage{textcomp}
\usepackage{url}
\usepackage{algorithmic}
\usepackage{algorithm}

\usepackage{ifthen}

\newcommand{\set}[1]{\mathopen{}\left\{#1\right\}}
\newcommand{\spr}[1]{\mathopen{}\left[#1\right]}
\newcommand{\pr}[1]{\mathopen{}\left(#1\right)}
\newcommand{\abs}[1]{\mathopen{}\left|#1\right|}

\newcommand{\enset}[2]{\set{{#1},\ldots,{#2}}}

\newcommand{\real}{\mathbb{R}}

\newcommand{\np}{\textbf{NP}}

\newcommand{\half}{\frac{1}{2}}

\newcommand{\ifam}[1]{\mathcal{#1}}
\newcommand{\negbord}[1]{\mathrm{negbord}\pr{#1}}

\newcommand{\funcdef}[3]{{#1}:{#2} \to {#3}}
\newcommand{\define}{\Leftarrow}

\newcommand{\cdf}{cdf}
\newcommand{\mean}[2]{\mathrm{E}_{#1}\spr{#2}}
\newcommand{\kl}[2]{\mathrm{KL}\pr{#1 \| #2}}
\newcommand{\ent}[1]{H\pr{#1}}

\newcommand{\pme}{p^{*}}
\newcommand{\pind}{p_{ind}}

\newcommand{\dtname}[1]{\emph{#1}}

\newcommand{\dispfunc}[2]{%
  \ensuremath{%
  \ifthenelse{\equal{\noexpand#2}{}}%
    {\mathrm{#1}}%
    {\mathrm{#1}\pr{#2}}}}

\newcommand{\condfunc}[3]{%
  \ensuremath{%
  \ifthenelse{\equal{\noexpand#3}{}}{%
    \ifthenelse{\equal{\noexpand#2}{}}%
      {{#1}}%
      {{#1}\pr{#2}}}%
    {{#1}\pr{#2; #3}}}}

\newcommand{\rank}[2][]{\condfunc{r}{#2}{#1}}
\newcommand{\rankI}{\rank[\ifam{I}]{G}}
\newcommand{\rankC}{\rank[\ifam{C}]{G}}
\newcommand{\rankA}{\rank[\ifam{A}]{G}}
\newcommand{\rankT}{\rank[\ifam{T}^*]{G}}
\newcommand{\rankF}{\rank[\ifam{F}^*]{G}}

\newcommand{\nrank}[2][]{\condfunc{nr}{#2}{#1}}
\newcommand{\nrankI}{\nrank[\ifam{I}]{G}}
\newcommand{\nrankC}{\nrank[\ifam{C}]{G}}
\newcommand{\nrankA}{\nrank[\ifam{A}]{G}}
\newcommand{\nrankT}{\nrank[\ifam{T}^*]{G}}
\newcommand{\nrankF}{\nrank[\ifam{F}^*]{G}}

\newcommand{\brin}[1]{r_b\pr{#1}}
\newcommand{\colstr}[1]{r_{cs}\pr{#1}}

\newtheorem{theorem}{Theorem}
\newtheorem{lemma}{Lemma}

\newtheorem{corollary}[theorem]{Corollary}

\newtheorem{example}[theorem]{Example}

\begin{document}
\title{Maximum Entropy Based Significance of Itemsets}
\author{Nikolaj Tatti\\
HIIT Basic Research Unit, Department of Computer Science\\
Helsinki University of Technology, Helsinki, Finland\\
ntatti@cc.hut.fi}

\maketitle
\thispagestyle{empty}
\begin{abstract}
We consider the problem of defining the significance of an itemset.  We say
that the itemset is significant if we are surprised by its frequency when
compared to the frequencies of its sub-itemsets. In other words, we estimate the
frequency of the itemset from the frequencies of its sub-itemsets and compute
the deviation between the real value and the estimate. For the estimation we
use Maximum Entropy and for measuring the deviation we use Kullback-Leibler
divergence.

A major advantage compared to the previous methods is that we are able to use
richer models whereas the previous approaches only measure the deviation from
the independence model.

We show that our measure of significance goes to zero for derivable itemsets
and that we can use the rank as a statistical test. Our empirical results
demonstrate that for our real datasets the independence assumption is too
strong but applying more flexible models leads to good results.
\end{abstract}

\section{Introduction}
How significant is a given itemset? Itemsets are popular and well-studied
patterns in binary data mining. The major drawback is that, given a dataset,
there are exponential number of itemsets. Hence, we need to rank itemsets in
order to prune the uninteresting ones.

Traditionally, the frequency of an itemset is used as a rank measure. The
higher the frequency, the more significant is the itemset. Frequency has many
virtues: It is easy to interpret and because of its property of
anti-monotonicity there exist efficient algorithms for finding all frequent
itemsets~\cite{agrawal93mining,agrawal96apriori}. There are, however, major
drawbacks. First, a frequent itemset may be insignificant: An itemset $AB$ may
be frequent just because itemsets $A$ and $B$ are frequent.  Second, an
infrequent itemset may be significant: If itemsets $A$ and $B$ are frequent,
the infrequency of $AB$ is interesting information.

Alternative methods for ranking itemsets are suggested
in~\cite{aggarwal98new,brin97beyond,dumouchel01empirical}. These methods are
discussed in more detail in Section~\ref{sec:related}. A common feature to
these methods is that they compare the frequency of an itemset to an estimate
obtained from the independence model. That is, the more the itemset deviates
from the independence model, the more surprising, and thus the more
significant, the itemset is.

Our proposal for ranking itemsets resembles the aforementioned approaches.  We
estimate the frequency of a given itemset from the frequencies of some selected
sub-itemsets. Namely, we use Maximum Entropy for the estimation.  This approach
is more flexible than the independence model, since the independence model uses
only the margins (the frequencies of itemsets of size $1$) for prediction
whereas our approach allows to use the information available from the itemsets
of larger size. While our ranking method is based on well-known tools, no
similar framework has been suggested previously.

Unlike the frequency, our measure is not decreasing with respect to set inclusion.
Hence we cannot mine significant itemsets in a level-wise fashion. However,
it turns out that in some cases we can prune a large set of uninteresting itemsets
(w.r.t. the measure). Namely, if the itemset is derivable~\cite{calders02mining},
then the measure is equal to $0$. We also point out that can be used as a
statistical test, thus providing a clear interpretation for the measure.

The rest of the paper is organized as follows: Preliminaries are given in
Section~\ref{sec:prel}. The definition and the properties of the measure are
given in Section~\ref{sec:theory}. We present related work in
Section~\ref{sec:related}. Section~\ref{sec:experiments} is devoted to
experiments and finally we provide conclusions in
Section~\ref{sec:conclusions}.

\section{Preliminaries and Notation}
\label{sec:prel}
In this section we review briefly theory of itemsets and also introduce some
notation that will be used later on.

A \emph{binary dataset} $D$ is a collection of $M$ binary vectors,
\emph{transactions}, having length $K$. Such dataset can be naturally
represented as a matrix of size $M \times K$. We denote the number of
transactions by $\abs{D} = M$.  To each column of the matrix we assign an
\emph{attribute} $a_i$. Let $A = \enset{a_1}{a_K}$ be the collection of all
attributes. An itemset $X \subseteq A$ is a set of attributes.

We say that a transaction (binary vector) $\omega$ \emph{covers} an itemset $X$
if $a_i \in X$ implies $\omega_i = 1$.  Given a dataset $D$, a \emph{frequency}
of an itemset $X$ is a proportion of the transaction in $D$ covering $X$.  Note
that if an itemset $Y$ is a subset of $X$, then the frequency of $Y$ is larger
than or equal to the frequency of $X$. In other words, frequency is decreasing
with respect to set inclusion.

A sample space $\Omega$ is the set of all binary vectors of length $K$.
We take a simplistic approach in defining distributions: A distribution
$\funcdef{p}{\Omega}{\spr{0,1}}$ is a function from a sample space $\Omega$ to
a real number between $0$ and $1$ such that $\sum_{\omega \in \Omega} p(\omega)
= 1$. Given an itemset $X$, a frequency of $X$ calculated from a distribution $p$
is the probability of binary vector covering $X$. We denote this by
\[
p(X = 1) = p(\omega \text{ covers } X).
\]

A family of itemsets $\ifam{F}$ is called \emph{anti-monotonic} or
\emph{downward closed} if every subset of each member of $\ifam{F}$ is also a
member of $\ifam{F}$. Note that a collection of $\sigma$-frequent itemsets,
that is, itemsets having frequency larger than some given threshold $\sigma$,
is downward closed. We are interested in three particular families: 

\begin{itemize}
\item $\ifam{I}$, the family containing only itemsets of size 1.
\item $\ifam{C}$, the family containing itemsets of size 1 and 2.
\item $\ifam{A}$, the family containing all itemsets.
\end{itemize}

A \emph{negative border} $\negbord{\ifam{F}}$ of the downward closed family
$\ifam{F}$ is the set of itemsets just above $\ifam{F}$. In other words, $X
\notin \ifam{F}$ is member of $\negbord{\ifam{F}}$ if there is no proper subset
$Y \subset X$ such that $Y \notin \ifam{F}$.

Given a dataset $D$, we say that an itemset $X$ is \emph{derivable} if by
knowing the frequencies (calculated from $D$) of each proper subset of $X$ we
can deduce the frequency of $X$. For example, if some subset of $X$ has a
frequency $0$, then we know that $X$ must also have frequency $0$. Thus, in
this case, $X$ is derivable. An itemset that is not derivable is called
\emph{non-derivable}. A family of all non-derivable itemsets is
downward closed~\cite{calders02mining}.

\section{Maximum Entropy Ranking}
\label{sec:theory}
In this section we introduce our ranking method and discuss its theoretical
properties.  The fundamental idea behind our approach is to measure how
surprising an itemset is compared to its subsets. In other words, we estimate
the itemset frequency by using the frequencies of its subsets and compare how
close is our estimation to the actual value. The estimation is done using
Maximum Entropy method and the comparison is done using Kullback-Leibler
divergence.

\subsection{Definition}

Let $D$ be a binary dataset and let $\enset{a_1}{a_K}$ be its attributes. The
number of columns in $D$ is $K$. Assume that we are given $G$, an itemset we
wish to rank. We define a projected dataset $D_G$ by keeping only the
attributes included in $G$.

Let $\Omega_G = \set{0,1}^{\abs{G}}$ be a space of binary vectors of length
$\abs{G}$. We define an \emph{empirical distribution}
$\funcdef{q_G}{\Omega_G}{\spr{0,1}}$ to be
\[
	q_G(\omega) = \frac{\text{Number of samples in $D_G$ equal to $\omega$}}{\abs{D_G}}.
\]

Our goal is to compare the distribution $q_G$ to a distribution obtained by
using Maximum Entropy~\cite{kullback68information}, a method that we will
describe next.

Assume now that we are given a family of itemsets $\ifam{F} \subseteq \ifam{A}$
and let $\theta_X$ be the frequency of $X \in \ifam{F}$ calculated from $D$.
Our next step is to define an approximative distribution using only the
itemsets in $\ifam{F}$. In defining $q_G$ we projected out the attributes
outside $G$. Similarly, we are only interested in subsets of $G$. Hence we
define a \emph{projected family} $\ifam{F}_G$ to be
\[
\ifam{F}_G = \set{X \in \ifam{F} \mid X \subset G, X \neq G, X \neq \emptyset}.
\]
Note that $\ifam{F}_G$ may contain $2^{\abs{G}} - 2$ itemsets, at maximum.
This is the case if $\ifam{F} = \ifam{A}$.

We say that a distribution $\funcdef{p}{\Omega_G}{\spr{0,1}}$ \emph{satisfies
the itemsets} $\ifam{F}_G$ if for each itemset $X \in \ifam{F}_G$ and its
frequency $\theta_X$ we have
\[
	p(X = 1) = \theta_X.
\]
Let $\mathbb{P}$ be the set of all distributions satisfying the itemsets
$\ifam{F}_G$.  This set is not empty since $q_G \in \mathbb{P}$. We select the
distribution from $\mathbb{P}$ maximizing the entropy
\[
	\ent{p} = -\sum_{\omega \in \Omega_G} p(\omega)\log p(\omega).
\]
We denote this distribution by $\pme$. Note that $\pme$ depends on $G$,
$\ifam{F}$, and $\theta$ but we have omitted these variables from the notation
for the sake of clarity.

We define the rank measure $\rank[\ifam{F}, D]{G}$ to be the divergence between
$q_G$ and $\pme$, that is,
\[
	\rank[\ifam{F}, D]{G} = \sum_{\omega \in \Omega_G} q_G(\omega)\log \frac{q_G(\omega)}{\pme(\omega)}.
\]
We omit $D$ from the notation when the dataset is clear from the context.
\begin{example}
Assume the simplest case where $G = a$ is an itemset of size $1$. Let $\theta_G$
be the frequency of $G$. Note that $\ifam{F}_G = \emptyset$, hence there are
no constraints on selecting $\pme$. This means that $\pme$ is the uniform
distribution, that is, $\pme(0) = \pme(1) = 1/2$. In this case the measure is
\[
	\rank[\ifam{F}]{a} = (1 - \theta_G)\log\pr{2(1 - \theta_G)} + \theta_G\log\pr{2\theta_G}
\]
obtaining its minimum when $\theta_G = 1/2$ and is at its maximum when
$\theta_G = 0$ or $\theta_G = 1$.
\end{example}

We are mainly interested in three kinds of measures: The first is {\rankI} in
which $\ifam{I}$ is the family of itemsets of size $1$. In this case the
Maximum Entropy distribution is equal to the independence model.

The second case is {\rankC}, where $\ifam{C}$ contains the itemsets of size $1$
and $2$. We can show that there exists a matrix $B$ such that for the non-zero
entries of $\pme$ we have
\[
	\pme(\omega) \propto \exp\pr{ \omega^T B \omega}.
\]
Hence, {\rankC} can be seen as the measure of the deviation from the discrete
Gaussian model.

Our third type of measure is {\rankA} in which $\pme$ is predicted from all
the proper sub-itemsets of $G$. In this case we can prove that for a certain
set of real numbers $r_i$ we have for the non-zero entries of $\pme$
\[
	\pme(\omega) \propto \prod_{X_i \in \ifam{A}_G}  \exp\pr{ r_i I\pr{\omega \text{ covers } X_i}},
\]
where $I$ is the indicator function.
We discuss the evaluation of our approach in Section~\ref{sec:computing}.

\subsection{Properties}
In this section we discuss various properties of $\rank{G}$.  We will first
point the connection between $\rank{G}$ and derivable itemsets and then discuss
the use of $\rank{G}$ as a statistical test.

\begin{theorem}
\label{thr:deriv1}
Let $G$ be a derivable itemset. Then \[\rank[\ifam{A}]{G} = 0.\]
\end{theorem}
\begin{proof}
We can argue that if we know the frequencies of all sub-itemsets of $G$, we
can derive the distribution $q_G$ and vice versa. This implies that there is
one-to-one correspondence between the distribution $p \in \mathbb{P}$
satisfying the itemsets $\ifam{A}_G$ and the frequency $p(G = 1)$. Since we can
derive the frequency of $G$ from $\ifam{A}_G$, it follows that $\mathbb{P} =
\set{q_G}$, and hence $\pme = q_G$.
\end{proof}

We can reformulate the previous theorem in a stronger form by pointing out that
we need to know only non-derivable itemsets.

\begin{theorem}
\label{thr:deriv2}
Let $\ifam{F}$ be a family of all non-derivable itemsets. Let $G$ be outside
of $\ifam{F}$. Then $\rank[\ifam{F}]{G} = 0$.
\end{theorem}
\begin{proof}
Since all unknown sub-itemsets of $G$ are derivable from $\ifam{F}_G$, the argument
of Theorem~\ref{thr:deriv1} holds. 
\end{proof}

The following theorem provides the interpretation to the value of $\rank{G}$ and
points out that we can use $\rank{G}$ as a statistical test.

\begin{theorem}
\label{thr:x2simple}
Let $G$ be a non-derivable itemset. Under the 0-hypothesis that $G$ is
distributed according to $\pme$, the quantity $2\abs{D}\rank[\ifam{A}]{G}$ is
distributed asymptotically as $\chi^2$ with degree $1$ of freedom.
\end{theorem}

Theorem~\ref{thr:x2simple} is a special case of the following more general
statement.

\begin{theorem}
\label{thr:x2generic}
Let $G$ be a non-derivable itemset and let $\ifam{F}$ be an itemset family.
Define $\ifam{H}$ to be
\[
	\ifam{H} = \set{X \in \ifam{A} \mid X \subseteq G, X \neq \emptyset, X \notin \ifam{F}_G},
\]
that is, $\ifam{H}$ is a family of sub-itemsets of $G$ not belonging to
$\ifam{F}_G$. Under the 0-hypothesis that the itemsets in $\ifam{H}$ are
distributed according to $\pme$, the quantity $2\abs{D}\rank[\ifam{F}]{G}$ is
distributed asymptotically as $\chi^2$ with degree $\abs{\ifam{H}} = 2^{\abs{G}}
- 1 - \abs{\ifam{F}_G}$ of freedom.
\end{theorem}
Theorem~\ref{thr:x2generic} is stated (but not proven) in a more general form
in~\cite{kullback68information}. A rather technical proof is provided in
Appendix~\ref{sec:assymptotic}.

Theorem~\ref{thr:x2generic} motivates us to define the \emph{normalised rank
measure} to be the one-sided $\chi^2$ test, that is,
\[
\nrank[\ifam{F}, D]{G} = \cdf\pr{2\,\abs{D}\,\rank[\ifam{F}, D]{G}},
\]
where $\cdf(a) = P\pr{\chi^2 < a}$ is the cumulative distribution function of $\chi^2$ with degree
$2^{\abs{G}} - 1 - \abs{\ifam{F}_G}$ of freedom. The number of degrees for
different rank measures are provided in Table~\ref{tab:summary}.

The following well-known result and its corollaries will play an important
role in solving the measures.

\begin{lemma}
Let $\pme$ be the Maximum Entropy distribution for itemsets $\ifam{F}$ and the
corresponding frequencies $\theta$. Let $q$ be a distribution satisfying the
itemsets $\ifam{F}$. Then we have
\[
-\sum_{\omega} q(\omega)\log \pme(\omega) = \ent{\pme}.
\]
\end{lemma}

\begin{corollary}
\label{cor:decompose1}
Let $\ifam{F}$ be the family of itemsets. We have that
\[
	\rank[\ifam{F}]{G} = \ent{\pme} - \ent{q_G},
\]
where $\pme$ is the Maximum Entropy distribution and $q_G$ is the empirical
distribution.
\end{corollary}

\begin{corollary}
\label{cor:decompose2}
Let $\ifam{F}$, $\ifam{H}$ be the families of itemsets such that $\ifam{H}
\subseteq \ifam{F}$.  Let $\pme_1$ be the Maximum Entropy distribution for
$\ifam{F}$ and let $\pme_2$ be the Maximum Entropy distribution for $\ifam{H}$.
We have that
\[
	\rank[\ifam{F}]{G} = \kl{q_G}{\pme_1} - \kl{\pme_1}{\pme_2},
\]
$q_G$ is the empirical distribution.
\end{corollary}

\begin{corollary}
\label{cor:monotone}
Let $\ifam{F}$, $\ifam{H}$ be the families of itemsets such that $\ifam{H}
\subseteq \ifam{F}$. We have that
\[
	\rank[\ifam{F}]{G} \leq \rank[\ifam{H}]{G}.
\]
\end{corollary}

\subsection{Flexible Models}

So far we have considered ranks with fixed families of itemsets. In this section
we introduce $2$ additional models. In these models the itemsets are selected
such that they minimise the rank.

Our first rank measure is the optimal tree model. A tree model can be described
as a tree defined on the attributes of $G$. The corresponding family $\ifam{T}$ of
itemsets contains the attributes from $G$ and the itemsets of size $2$ corresponding
to the edges of the tree.

\begin{example}
Consider $G = \set{a, b, c, d, e}$ and consider the tree given in
Figure~\ref{fig:tree}. The corresponding family of itemsets is $\ifam{T} =
\set{a, b, c, d, e, ab, ac, ad, de}$.
\begin{figure}
\centering
\includegraphics[scale=0.45]{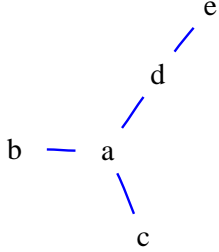}
\label{fig:tree}
\caption{A toy tree model. The related itemsets $\set{a, b, c, d, e, ab, ac, ad, de}$
correspond to the attributes and the edges of the tree.}
\end{figure}
\end{example}

We can show that the Maximum Entropy distribution for $\ifam{T}$ has the form
\[
\pme = \prod_{\set{a,b} \in \ifam{T}} \pme\pr{a,b} / \prod_{a \in G} \pme\pr{a}.
\]
This is, of course, Chow-Liu tree model\cite{chow68tree}. We define the optimal tree
to be the ne that minimises the rank, that is,
\[
\ifam{T}^* = \underset{\ifam{T} \text{ is a tree}}{\arg \min} \rank[\ifam{T}, D]{G}.
\]

To solve this tree let $\pind$ be the independence distribution.
Corollary~\ref{cor:decompose1} allows us to rewrite the rank measure as
\[
\rank[\ifam{T}]{G} = \kl{q_D}{\pme} = \kl{q_D}{\pind} - \kl{\pme}{\pind}.
\]
Note that the first term $\kl{q_D}{\pind}$ does not depend on $\ifam{T}$. Hence
we need to maximise the second term $\kl{\pme}{\pind}$. This is the mutual information
of the tree and maximising this term is equivalent to finding maximum spanning tree
in the mutual information graph. This can be done in polynomial time~\cite{chow68tree}.

There is a deep connection between the rank $\rank{G; \ifam{T}}$ and the rank
for D-trees suggested in~\cite{heikinheimo07entropy}. We can rewrite, by
applying Corollary~\ref{cor:decompose1}, the rank as
\[
\rank[\ifam{T}]{G} = \kl{q_D}{\pme} = \ent{\pme} - \ent{q_G}.
\]
The first term $\ent{\pme}$ is the rank that is used in~\cite{heikinheimo07entropy}.
The authors in~\cite{heikinheimo07entropy} seek patterns that have small $\ent{\pme}$,
that is, trees that have strong dependencies between the attributes,
whereas we are interested in patterns that produce large $\rank[\ifam{T}^*]{G}$,
sets of attributes whose joint distribution cannot be explained even by the best
tree model.

Our second model involves in finding a downward closed family $\ifam{F}$ of
itemsets that produces the smallest normalised rank. Note that
Corollary~\ref{cor:monotone} implies that the rank decreases when we increase
the number of known itemsets. However, this does not hold for the normalised
rank and we will see that, contrary to the expectations, the best model can be
different than $\ifam{A}_G$, the set of all sub-itemsets of $G$.  In other
words, knowing all sub-itemsets does not guarantee the best model but, in fact,
itemsets of higher order may mislead the prediction.

Unlike with the tree models, to our knowledge, there is no polynomial algorithm
for finding the optimal downward closed family. Hence, we suggest a simple
greedy approach. We start from the itemsets of size $1$ and select the itemset
from the negative border that minimises the rank. The itemset is added into the
family and the procedure is repeat until there is no itemset that can decrease
the rank. The algorithm is stated in Algorithm~\ref{algo:greedy}. We use $\ifam{F}^*$
to denote the resulting family.

\begin{algorithm}[ht!]
\caption{Greedy algorithm for finding the optimal downward closed family of
itemsets. The input is the data set $D$ and the query itemset $G$. The output
is $\ifam{F}^*$ a family of itemsets that produces low rank for the itemset.}
\begin{algorithmic}
\STATE $\ifam{F}^* \define \ifam{I}_G$. \COMMENT{Initialise $\ifam{F}^*$ with itemsets of size $1$.}
\REPEAT
\STATE $Y \define \underset{X \in \negbord{\ifam{F}^*}}{\arg \min} \nrank[\ifam{F}^* \cup X]{G}$.
\IF{$\nrank[\ifam{F}^* \cup Y]{G} < \nrank[\ifam{F}^*]{G}$}
\STATE $\ifam{F}^* \define \ifam{F}^* \cup Y$.
\ENDIF
\UNTIL{no more changes in $\ifam{F}^*$.}
\end{algorithmic}
\label{algo:greedy}
\end{algorithm}

\subsection{Computing Rank}
\label{sec:computing}
Corollary~\ref{cor:decompose1} allows us to rewrite the rank as a difference of
two entropies
\[
\rank{G} = \kl{q_G}{\pme} = \ent{\pme} - \ent{q_G}.
\]
Both distributions have $\abs{\Omega_G} = 2^{\abs{G}}$ entries. However,
the distribution $q_G$ can have only $\abs{D}$ positive entries at
maximum, hence the term $\ent{q_G}$ can be computed efficiently.

The challenge in calculating the measure is to solve the Maximum Entropy
distribution $\pme$ and calculate its entropy. This can be done in polynomial
time for the independence model and for the tree models.  However, in the
general case solving $\pme$ is an \np-complete
problem~\cite{tatti06complexity,cooper90complexity}; In such cases the
distribution is solved using Iterative Scaling
algorithm~\cite{darroch72gis,jirousek95iterative}.  The algorithm consists of
consecutive steps. One such step requires $O\pr{\abs{\Omega_G}} =
O\pr{2^{\abs{G}}}$ time. Hence computing the measure requires exponential time
but it is doable for itemsets of reasonable size. The summary for evaluation
times is provided in Table~\ref{tab:summary}.

\begin{table}
\centering
\begin{tabular}{rrrr}
\toprule
Measure & Description & \# of degrees & Evaluation time \\
\midrule
$\rank[\ifam{I}]{G}$ & Independence model     & $2^{\abs{G}} - 1 - \abs{G}$  & $O\pr{\abs{G}}$ \\
$\rank[\ifam{C}]{G}$ & Gaussian model         & $2^{\abs{G}} - 1 - \half\abs{G}\pr{\abs{G} + 1}$ & $O\pr{2^{\abs{G}}}$ per iter.\\
$\rank[\ifam{A}]{G}$ & All subsets model      & $1$ & $O\pr{2^{\abs{G}}}$ per iter.\\
$\rank[\ifam{T}^*]{G}$ & Optimal tree model   & $2^{\abs{G}} - 2\abs{G}$ & $O\pr{\abs{G}^2}$ \\
$\rank[\ifam{F}^*]{G}$ & Optimal family model & $2^{\abs{G}} - 1 - \abs{\ifam{F}}$ & $O\pr{8^{\abs{G}}}$ per iter.\\
\bottomrule
\end{tabular}
\label{tab:summary}
\caption{Summary of the rank measure. The number of degrees, the third column,
is used as a a parameter for $\chi^2$ distribution, when computing the normalised rank.
The fourth column represents the evaluation times for the entropy of $\pme$.}
\end{table}

\paragraph{The effect of pruning itemsets.}
Note that in defining the measure we only use itemsets that are subsets of the
query itemset $G$. This pruning guarantees that the number of entries in the
distributions is $2^{\abs{G}}$ and not, at worst, $2^K$, where $K$ is the
number of columns in the dataset. Pruning attributes is essential since solving
$\pme$ is exponential to the number of attributes. The downside is that pruning
may change the prediction as the following example demonstrates.

\begin{example}
Assume that we have $3$ attributes, $a$, $b$, and $c$. Our known itemsets are
$\ifam{F} = \set{a, b, c, ac, bc}$ and their frequencies are $\theta_a =
\theta_b = \theta_c = \theta_{ac} = \theta_{bc} = 1/2$. In other words, the
attributes are indentical and correspond to a fair coin flip. Assume that we are
interested in rank of $G = ab$. In this case the pruned family of itemsets
is $\ifam{F}_G = \set{a, b}$ and the Maximum Entropy distribution is the
uniform distribution. The empirical distribution is
\[
\begin{split}
q_G(a = 0, b = 0) = q_G(a = 1, b = 1) & = 1/2 \\ 
q_G(a = 1, b = 0) = q_G(a = 0, b = 1) &= 0.
\end{split}
\]
The rank is then $\rank[\ifam{F}]{ab} = 0.69$. However, if we had used the
frequencies of $ac$ and $bc$, we would have concluded that $a = b$ and that the
Maximum Entropy distribution is equal to the empirical distribution, hence the
rank would have been $0$.
\end{example}

In~\cite{tatti06projections} we investigate the effect of pruning attributes
and conclude that in some cases we can remove a large portion of attributes
outside $G$. However, in those cases, the family of known itemsets has many
restrictions and, for instance, we cannot remove safely any attribute from the
gaussian model.

\section{Related Work}
\label{sec:related}
Traditionally, the support (frequency) of the itemset is used for ranking
itemsets. Alternative measures that resemble the support are studied
in~\cite{omiecinski03measure}.

Our work resembles approach of~\cite{brin97beyond} in which the authors defined
the significance of an itemset by comparing the distribution $q_G$ against the
independence model. The authors used $\chi^2$ statistical test as a measure,
that is, if $p$ is the distribution related to the independence model, the rank
measure is
\begin{equation}
\brin{G} = \sum_{\omega \in \Omega_G} \frac{\pr{q_G(\omega) - p(\omega)}^2}{p(\omega)}.
\label{eq:brin}
\end{equation}
In~\cite{dumouchel01empirical} the authors also compare the frequency of an
itemset against the independence model but in addition they use Bayes screening
to smooth the values. Also, in~\cite{aggarwal98new} the authors proposed the
collective strength as a measure of significance. To be more specific, we say
that a transaction $\omega \in \Omega_G$ is \emph{good} if it contains only
$0$s or only $1$s.  Let $p$ be the distribution related to the independence
model. Then the measure is
\begin{equation}
\colstr{G} = \frac{q_G\pr{\omega \text{ is good}}}{p\pr{\omega \text{ is good}}}
\frac{p\pr{\omega \text{ is bad}}}{q_G\pr{\omega \text{ is bad}}}
\label{eq:colstr}.
\end{equation}
This measure obtains small values when data obeys the independence model. In a
related work presented in~\cite{dong99efficient} the authors define an itemset
to be interesting if its frequency increases significantly from one dataset to
another. In~\cite{gallo07mini} the authors order itemsets based on their
p-values.  In~\cite{heikinheimo07entropy} the authors used entropy of tree
models for ranking itemsets. In addition, many measures has been suggested for
ranking association
rules~\cite{piatetsky91rules,brin97itemsets,agrawal93mining,jaroszewicz02pruning}.

The authors in~\cite{pavlov03beyond} showed empirically that Maximum Entropy
model provides excellent estimates for itemsets. Rank can be used for pruning a
large family of itemsets by picking the itemsets having the largest rank.
Other pruning methods are proposed
in~\cite{boulicaut00approximation,calders02mining,pasquier99discovering}. The
authors in~\cite{webb06significant} suggest a generic framework for discovering
significant rules. In addition, a relevant framework is described
in~\cite{mielikainen03pattern}; the authors define a pattern ordering given an
estimation algorithm and a loss function. In~\cite{noren06who} the authors use
information component analysis to find patterns in a drug safety database.

\section{Experiments}
\label{sec:experiments}
In this section we present our empirical results. In the first $3$ sections we
explain the datasets and the setup. In our experiments we investigate the
significance of itemsets, how different measures are related to each other, and
the monotonicity of the ranks.

\subsection{Synthetic Datasets}
For the testing purposes we created two synthetic datasets. Each dataset
contained $100$ attributes and $5000$ rows. The first dataset,
\dtname{gen-ind}, was generated such that the attributes were independent.
The margins were sampled uniformly from $\spr{0,1}$. In the second dataset,
\dtname{gen-copy}, each column was a copy of the previous column corrupted
by the symmetric white noise. The amount of noise, that is the probability
\[
    p\pr{a_i = 1 \mid a_{i - 1} = 0} = p\pr{a_i = 0 \mid a_{i - 1} = 1},
\]
was selected uniformly from $\spr{0,1}$ for each column $a_i$, individually.
The first column was generated by a coin flip. Our expectations are that
in \dtname{gen-ind} the itemsets of size $1$ are significant and that
in \dtname{gen-copy} the itemsets of size $2$ are significant.

\subsection{Real Datasets}
In our experiments we used the following real-world datasets. Data in
\dtname{Accidents}\footnote{\url{http://fimi.cs.helsinki.fi/data/accidents.dat.gz}}
were obtained from the Belgian ``Analysis Form for Traffic Accidents'' forms
that is filled out by a police officer for each traffic accident that occurs
with injured or deadly wounded casualties on a public road in Belgium. In
total, $340\,183$ traffic accident records are included in the
dataset~\cite{geurts03accidents}.  The datasets
\dtname{POS}\footnote{\url{http://www.ecn.purdue.edu/KDDCUP/data/BMS-POS.dat.gz}},
\dtname{WebView-1}\footnote{\url{http://www.ecn.purdue.edu/KDDCUP/data/BMS-WebView-1.dat.gz}}
and
\dtname{WebView-2}\footnote{\url{http://www.ecn.purdue.edu/KDDCUP/data/BMS-WebView-2.dat.gz}}
were contributed by Blue Martini Software as the KDD Cup 2000
data~\cite{kohavi00bms}. \dtname{POS} contains several years worth of
point-of-sale data from a large electronics retailer.  \dtname{WebView-1} and
\dtname{WebView-2} contain several months worth of click-stream data from two
e-commerce web sites.
\dtname{Kosarak}\footnote{\url{http://fimi.cs.helsinki.fi/data/kosarak.dat.gz}}
consists of (anonymised) click-stream data of a Hungarian on-line news portal.
\dtname{Retail}\footnote{\url{http://fimi.cs.helsinki.fi/data/retail.dat.gz}}
is a retail market basket data supplied by an anonymous Belgian retail
supermarket store~\cite{brijs99retail}.  The dataset
\dtname{Paleo}\footnote{NOW public release 030717 available
from~\cite{fortelius05now}.}  contains information of species fossils found in
specific paleontological sites in Europe~\cite{fortelius05now}, preprocessed
as in~\cite{fortelius06spectral}.

\subsection{Setup for the Experiments}
\label{sec:setup}
In this section we will describe how we conducted our experiments. We reduced
the largest datasets by selecting the first $10000$ rows and $200$ most
frequent attributes. From each dataset we computed all \emph{almost
non-derivable} itemsets. By almost non-derivable we mean that the difference
between the upper bound and the lower bound of a given itemset, say $G$, is at
least $n$ transactions. In other words, if we know the frequencies of all
sub-itemsets of $G$, then we cannot predict the frequency of $G$ within $n$
transactions. If $n = 0$, then an itemset is non-derivable. It is known that
the family of almost non-derivable itemsets is
anti-monotonic~\cite[Lemma~3.1]{calders02mining}. A reason to use almost
non-derivable itemsets instead of frequent itemsets is the statement of
Theorem~\ref{thr:deriv2}, that is, $\rankA = 0$ if the itemset is derivable.
The other reason is that we want to study how the measure behaves for
infrequent itemsets.

To keep the sizes of the obtained families within reasonable bounds we used
different thresholds for different datasets: For \dtname{gen-ind},
\dtname{Retail} and \dtname{WebView-2} we set $n = 5$. For \dtname{POS} the
threshold $n$ was set to $10$ and for \dtname{gen-copy} and \dtname{Accidents}
$n$ was set to $100$. For the rest of the datasets we set $n = 0$, that is, we
mined all non-derivable itemsets from these datasets.

For each itemset from the obtained itemsets we queried the following measures:
\begin{itemize}
\item Frequency.
\item Normalised rank measures {\nrankI}, {\nrankC}, {\nrankA}, {\nrankT}, {\nrankF}.
\item Measures discussed in Section~\ref{sec:related}: A $\chi^2$ test
$\brin{G}$ defined in Eq.~\ref{eq:brin} and a collective strength $\colstr{G}$
defined in Eq.~\ref{eq:colstr}.
\end{itemize}

The evaluation times and the sizes of the query families are given in
Table~\ref{tab:query_stats}.

\begin{table}
\center
\begin{tabular}{r@{\hspace{0.2cm}}r@{\hspace{0.2cm}}r@{\hspace{0.2cm}}r r@{\hspace{0.2cm}}r@{\hspace{0.2cm}}r@{\hspace{0.2cm}}r@{\hspace{0.2cm}}r@{\hspace{0.2cm}}r}
\toprule
& & & & \multicolumn{5}{c}{Evaluation times} \\
\cmidrule{5-9}
Data & $n$ & \# of $G$ & $\max \abs{G}$ & {\nrankI}& {\nrankC}& {\nrankA}& {\nrankT}& {\nrankF} \\
\midrule
\dtname{gen-ind} & $5$ & $156699$ & $6$ & $2s$ & $52s$ & $29min$ & $2s$ & $11min$ \\
\dtname{gen-copy} & $100$ & $111487$ & $4$ & $0s$ & $12s$ & $57s$ & $0s$ & $1min$ \\
\dtname{Accidents} & $100$ & $354399$ & $6$ & $2s$ & $1min$ & $19min$ & $3s$ & $108min$ \\
\dtname{Kosarak} & $5$ & $223734$ & $5$ & $1s$ & $4s$ & $9s$ & $0s$ & $47s$ \\
\dtname{Paleo} & $0$ & $166903$ & $5$ & $0s$ & $8s$ & $35s$ & $0s$ & $2min$ \\
\dtname{POS} & $10$ & $246640$ & $6$ & $1s$ & $8s$ & $27s$ & $1s$ & $5min$ \\
\dtname{Retail} & $0$ & $818813$ & $6$ & $3s$ & $19s$ & $49s$ & $4s$ & $4min$ \\
\dtname{WebView-1} & $5$ & $226313$ & $5$ & $1s$ & $5s$ & $8s$ & $1s$ & $39s$ \\
\dtname{WebView-2} & $0$ & $715398$ & $6$ & $3s$ & $27s$ & $2min$ & $4s$ & $11min$ \\
\bottomrule
\end{tabular}
\caption{The evaluation times and the sizes of the query families. The second
column is the threshold used in mining almost non-derivable itemsets. The fourth
column is the maximal size of a query itemset. The evaluation time does not include
the time spent mining itemsets.}
\label{tab:query_stats}
\end{table}

\subsection{Significant Itemsets}
\label{sec:significant}
Our first experiment is to study how many of the itemsets are significant. We
did this by comparing the our rank measures with risk level 0.05. The results
are given in Tables~\ref{tab:sig_table_ind}--\ref{tab:sig_table_flex}. We also
provide a typical example of box plots in Figure~\ref{fig:box_example}.

\begin{table}
\center
\begin{tabular}{rrrrrrrr}
\toprule
& \multicolumn{6}{c}{itemset size} \\
\cmidrule{2-7}
Data & 1 & 2 & 3 & 4 & 5 & 6 & All \\
\midrule
\dtname{gen-ind} & $.92$ & $.05$ & $.04$ & $.03$ & $.02$ & $.01$ & $.03$ \\
\dtname{gen-copy} & $.08$ & $.14$ & $.24$ & $.03$ & -- & -- & $.07$ \\
\dtname{Accidents} & $.99$ & $.60$ & $.95$ & $1$ & $1$ & $1$ & $.97$ \\
\dtname{Kosarak} & $1$ & $.62$ & $.99$ & $1$ & $1$ & -- & $.96$ \\
\dtname{Paleo} & $1$ & $.30$ & $.81$ & $.99$ & $1$ & -- & $.88$ \\
\dtname{POS} & $1$ & $.45$ & $.99$ & $1$ & $1$ & $1$ & $.95$ \\
\dtname{Retail} & $1$ & $.14$ & $.30$ & $.93$ & $1$ & $1$ & $.45$ \\
\dtname{WebView-1} & $1$ & $.70$ & $1$ & $1$ & $1$ & -- & $.97$ \\
\dtname{WebView-2} & $1$ & $.20$ & $.69$ & $1$ & $1$ & $1$ & $.85$ \\
\bottomrule
\end{tabular}
\caption{The percentages of significant itemsets according to {\nrankI}.  Each
entry is a fraction of itemsets of specific size calculated from a specific
dataset.  Significance is measured using $\chi^2$ distribution with $0.05$ risk
level.}
\label{tab:sig_table_ind}
\end{table}

\begin{table}
\center
\begin{tabular}{r
r@{\hspace{0.2cm}}r@{\hspace{0.2cm}}r@{\hspace{0.2cm}}r@{\hspace{0.2cm}}r@{\hspace{0.2cm}}r@{\hspace{0.2cm}}r@{\hspace{0.2cm}}c
r@{\hspace{0.2cm}}r@{\hspace{0.2cm}}r@{\hspace{0.2cm}}r@{\hspace{0.2cm}}r@{\hspace{0.2cm}}r@{\hspace{0.2cm}}r}
\toprule
& \multicolumn{7}{c}{{\nrankC}, itemset size} & & \multicolumn{7}{c}{{\nrankA}, itemset size}\\
\cmidrule{2-8}
\cmidrule{10-16}
Data & 1 & 2 & 3 & 4 & 5 & 6 & All & & 1 & 2 & 3 & 4 & 5 & 6 & All \\
\midrule
\dtname{gen-ind} & $.92$ & $.05$ & $.06$ & $.05$ & $.04$ & $.03$ & $.05$ &  & $.92$ & $.05$ & $.06$ & $.06$ & $.06$ & $.06$ & $.06$ \\
\dtname{gen-copy} & $.08$ & $.14$ & $.06$ & $.03$ & -- & -- & $.03$ &  & $.08$ & $.14$ & $.06$ & $.05$ & -- & -- & $.05$ \\
\dtname{Accidents} & $.99$ & $.60$ & $.21$ & $.45$ & $.62$ & $.60$ & $.45$ &  & $.99$ & $.60$ & $.21$ & $.07$ & $.05$ & $.06$ & $.11$ \\
\dtname{Kosarak} & $1$ & $.62$ & $.32$ & $.50$ & $.38$ & -- & $.37$ &  & $1$ & $.62$ & $.32$ & $.10$ & $.04$ & -- & $.33$ \\
\dtname{Paleo} & $1$ & $.30$ & $.12$ & $.15$ & $.21$ & -- & $.15$ &  & $1$ & $.30$ & $.12$ & $.21$ & $.64$ & -- & $.18$ \\
\dtname{POS} & $1$ & $.45$ & $.09$ & $.21$ & $.43$ & $.66$ & $.17$ &  & $1$ & $.45$ & $.09$ & $.06$ & $.07$ & $.05$ & $.11$ \\
\dtname{Retail} & $1$ & $.14$ & $.04$ & $.08$ & $.12$ & $.38$ & $.05$ &  & $1$ & $.14$ & $.04$ & $.15$ & $.27$ & $.25$ & $.07$ \\
\dtname{WebView-1} & $1$ & $.70$ & $.48$ & $.32$ & $.52$ & -- & $.48$ &  & $1$ & $.70$ & $.48$ & $.09$ & $.07$ & -- & $.45$ \\
\dtname{WebView-2} & $1$ & $.20$ & $.11$ & $.20$ & $.88$ & $1$ & $.17$ &  & $1$ & $.20$ & $.11$ & $.16$ & $.36$ & $.48$ & $.15$ \\
\bottomrule
\end{tabular}
\caption{The percentages of significant itemsets according to {\nrankC} and {\nrankA}. Each
entry is a fraction of itemsets of specific size calculated from a specific
dataset. Significance is measured using $\chi^2$ distribution with $0.05$ risk
level.}
\label{tab:sig_table_cov_all}
\end{table}

\begin{table}
\center
\begin{tabular}{r
r@{\hspace{0.2cm}}r@{\hspace{0.2cm}}r@{\hspace{0.2cm}}r@{\hspace{0.2cm}}r@{\hspace{0.2cm}}r@{\hspace{0.2cm}}r@{\hspace{0.2cm}}c
r@{\hspace{0.2cm}}r@{\hspace{0.2cm}}r@{\hspace{0.2cm}}r@{\hspace{0.2cm}}r@{\hspace{0.2cm}}r@{\hspace{0.2cm}}r}
\toprule
& \multicolumn{7}{c}{{\nrankT}, itemset size} & & \multicolumn{7}{c}{{\nrankF}, itemset size}\\
\cmidrule{2-8}
\cmidrule{10-16}
Data & 1 & 2 & 3 & 4 & 5 & 6 & All & & 1 & 2 & 3 & 4 & 5 & 6 & All \\
\midrule
\dtname{gen-ind} & $.92$ & $.05$ & $.02$ & $.01$ & $.01$ & $0$ & $.01$ &  & $.92$ & $.05$ & $.01$ & $.01$ & $0$ & $0$ & $.01$ \\
\dtname{gen-copy} & $.08$ & $.14$ & $.02$ & $0$ & -- & -- & $.01$ &  & $.08$ & $.14$ & $.01$ & $0$ & -- & -- & $.01$ \\
\dtname{Accidents} & $.99$ & $.60$ & $.40$ & $.80$ & $.95$ & $.97$ & $.75$ &  & $.99$ & $.60$ & $.18$ & $.12$ & $.13$ & $.02$ & $.15$ \\
\dtname{Kosarak} & $1$ & $.62$ & $.80$ & $.94$ & $1$ & -- & $.80$ &  & $1$ & $.62$ & $.32$ & $.05$ & $.03$ & -- & $.32$ \\
\dtname{Paleo} & $1$ & $.30$ & $.10$ & $.35$ & $.81$ & -- & $.24$ &  & $1$ & $.30$ & $.06$ & $.05$ & $.04$ & -- & $.07$ \\
\dtname{POS} & $1$ & $.45$ & $.47$ & $.99$ & $1$ & $1$ & $.65$ &  & $1$ & $.45$ & $.08$ & $.02$ & $.02$ & $0$ & $.09$ \\
\dtname{Retail} & $1$ & $.14$ & $.03$ & $.18$ & $.78$ & $1$ & $.07$ &  & $1$ & $.14$ & $.01$ & $.02$ & $.02$ & $.13$ & $.02$ \\
\dtname{WebView-1} & $1$ & $.70$ & $.83$ & $1$ & $1$ & -- & $.84$ &  & $1$ & $.70$ & $.46$ & $.07$ & $.30$ & -- & $.43$ \\
\dtname{WebView-2} & $1$ & $.20$ & $.11$ & $.57$ & $1$ & $1$ & $.37$ &  & $1$ & $.20$ & $.06$ & $.03$ & $.14$ & $.44$ & $.05$ \\
\bottomrule
\end{tabular}
\caption{The percentages of significant itemsets according to {\nrankT} and {\nrankF}.  Each
entry is a fraction of itemsets of specific size calculated from a specific
dataset.  Significance is measured using $\chi^2$ distribution with $0.05$ risk
level.}
\label{tab:sig_table_flex}
\end{table}

\begin{figure}
\center
\includegraphics[width=0.2\textwidth]{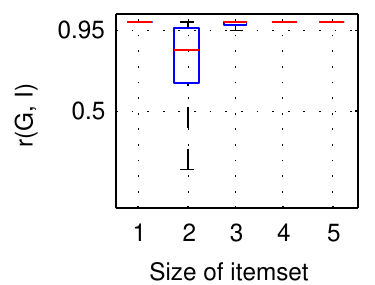}%
\includegraphics[width=0.2\textwidth]{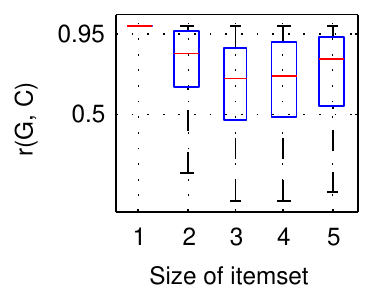}%
\includegraphics[width=0.2\textwidth]{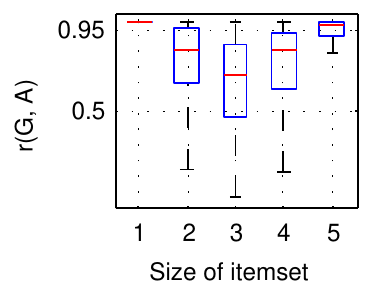}%
\includegraphics[width=0.2\textwidth]{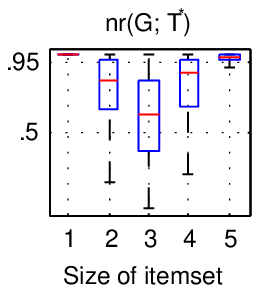}%
\includegraphics[width=0.2\textwidth]{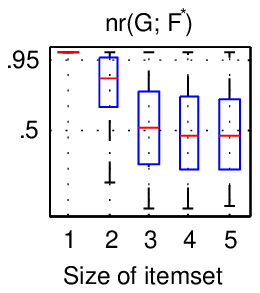}
\caption{Box plots of the rank measures computed from \dtname{Paleo}.}
\label{fig:box_example}
\end{figure}

Let us first study \dtname{gen-ind}, a synthetic dataset with independent
columns.  We see from Table~\ref{tab:sig_table_ind} that according to {\nrankI}
a large portion of itemsets of size $1$ are significant but only a small
portion of itemsets having size larger than $1$ is significant. This is an
expected result since the frequencies obey the independence model. In
Tables~\ref{tab:sig_table_cov_all} we have similar results for {\nrankC} and
for {\nrankA}. However, the values of {\nrankC} and for {\nrankA} tend to be
larger than the values of {\nrankI}.  The reason for this is a type of
overlearning: Since the frequencies of itemsets are calculated from the
datasets, they are imprecise. Hence, the itemsets with larger size mislead us
during prediction, because the resulting Maximum Entropy distribution is not an
independent model (although close to one).

Let us continue by studying \dtname{gen-copy}, a synthetic data in which an
attribute is a noisy copy of the previous attribute. We see that {\nrankT}
tends to have smaller ranks than {\nrankI} when $G$ has size $3$. The reason
for this is that, unlike with \dtname{gen-ind}, the independence model cannot
explain the dataset. However, when we predict using also the itemsets of size
$2$, the prediction becomes more accurate. The measures {\nrankC} and {\nrankA}
also produce small ranks, however, these ranks tend to be slightly larger than
the ranks of {\nrankT}.

We turn our attention to real datasets. We see that for these datasets the
independence model is too strict: According to {\nrankI} almost all itemsets
are significant: The results change drastically, when we use richer models.
According to {\nrankC} or {\nrankA} only $5\%$--$50\%$ of the itemsets are
significant, depending on the dataset.  Similar overfitting that occurred with
\dtname{gen-ind} also occurs in some but not all real datasets (see
Figure~\ref{fig:box_example}). For instance, in \dtname{Retail} {\nrankA} tends
to produce higher values than {\nrankC} but not in \dtname{POS}.

\subsection{The Effect of the Known Itemsets}
We continued our experiments by comparing the measures {\nrankI}, {\nrankC},
{\nrankA}, {\nrankT}, and {\nrankF} against each other. This was done by
calculating the correlations between the rank measures. The results are given
in Tables~\ref{tab:corrs1}~and~\ref{tab:corrs_flex}.

\begin{table}
\center
\begin{tabular}{rccc}
\toprule
& {\nrankI} & {\nrankI} & {\nrankC} \\
& vs. & vs. & vs. \\
Data & {\nrankC} & {\nrankA} & {\nrankA} \\
\midrule
\dtname{gen-ind} & $.74$ & $.26$ & $.39$ \\
\dtname{gen-copy} & $.52$ & $.28$ & $.53$ \\
\dtname{Accidents} & $.17$ & $.07$ & $.37$ \\
\dtname{Kosarak} & $.14$ & $.13$ & $.90$ \\
\dtname{Paleo} & $.16$ & $.22$ & $.67$ \\
\dtname{POS} & $.12$ & $.10$ & $.77$ \\
\dtname{Retail} & $.62$ & $.67$ & $.88$ \\
\dtname{WebView-1} & $.14$ & $.12$ & $.88$ \\
\dtname{WebView-2} & $.43$ & $.51$ & $.68$ \\
\bottomrule
\end{tabular}
\caption{Correlations between the measures {\nrankI}, {\nrankC}, and {\nrankA}.}
\label{tab:corrs1}
\end{table}

\begin{table}
\center
\begin{tabular}{r
r@{\hspace{0.2cm}}r@{\hspace{0.2cm}}r@{\hspace{0.2cm}}r@{\hspace{0.2cm}}
r@{\hspace{0.2cm}}r@{\hspace{0.2cm}}r@{\hspace{0.2cm}}r}
\toprule
& \multicolumn{3}{c}{{\nrankT} vs.} & & \multicolumn{4}{c}{{\nrankF} vs.} \\
\cmidrule{2-4} \cmidrule{6-9}
Data & {\nrankI} & {\nrankC} & {\nrankA} & & {\nrankI} & {\nrankC} & {\nrankA} & {\nrankT}\\
\midrule
\dtname{gen-ind} & $.81$ & $.94$ & $.37$ &  & $.81$ & $.91$ & $.36$ & $.98$ \\
\dtname{gen-copy} & $.62$ & $.92$ & $.49$ &  & $.64$ & $.89$ & $.49$ & $.97$ \\
\dtname{Accidents} & $.34$ & $.57$ & $.20$ &  & $.07$ & $.58$ & $.54$ & $.29$ \\
\dtname{Kosarak} & $.41$ & $.41$ & $.36$ &  & $.13$ & $.85$ & $.93$ & $.46$ \\
\dtname{Paleo} & $.35$ & $.64$ & $.47$ &  & $.16$ & $.84$ & $.67$ & $.62$ \\
\dtname{POS} & $.42$ & $.34$ & $.28$ &  & $.05$ & $.72$ & $.84$ & $.21$ \\
\dtname{Retail} & $.66$ & $.82$ & $.82$ &  & $.59$ & $.92$ & $.84$ & $.86$ \\
\dtname{WebView-1} & $.56$ & $.29$ & $.24$ &  & $.12$ & $.86$ & $.93$ & $.28$ \\
\dtname{WebView-2} & $.59$ & $.62$ & $.65$ &  & $.34$ & $.77$ & $.72$ & $.54$ \\
\bottomrule
\end{tabular}
\caption{Correlations between the flexible measures {\nrankT} and {\nrankF} and the
measures {\nrankI}, {\nrankC}, and {\nrankA}.}
\label{tab:corrs_flex}
\end{table}

From the results we see that all correlations are positive. For the real
datasets the correlations between {\nrankC} and {\nrankA} are systematically
higher than the correlations between {\nrankI} and {\nrankA} or between
{\nrankC} and {\nrankA}. This suggests that {\nrankI} produces different ranks
whereas {\nrankC} and {\nrankA} are more similar. This supports the behaviour
we have seen in Section~\ref{sec:significant}.

The measure {\nrankF} correlate more with {\nrankA} and {\nrankC} than with
{\nrankI}. The correlation between {\nrankF} and {\nrankT} is somewhat weaker
but it is stronger than the correlation between {\nrankF} and {\nrankI}.

\subsection{Flexible Models}

Our next goal is to compare the flexible measures {\nrankT} and {\nrankF}
against the rest of the measures. From Table~\ref{tab:sig_table_flex} we see
that {\nrankF} tend to produce the smallest amount of significant itemsets
whereas the {\nrankT} produces large ranks, especially for queries with many
attributes.

We calculated the number of queries in which {\nrankT} and {\nrankF} produce
smaller rank than the rest of the measures. Since the measures are equivalent
for the queries of size $1$ and $2$, these queries were ignored.  From the
results given in Table~\ref{tab:perf_flex} we see that the flexible models
outperform {\nrankI}, however, the performance against other measure depends on
the data set. For instance, {\nrankF} outperform {\nrankC} and {\nrankA} in
\dtname{Retail} but produces larger ranks in \dtname{Kosarak}.  This suggests
that the greedy algorithm sometimes fails to find the optimal family
$\ifam{F}^*$.

\begin{table}
\center
\begin{tabular}{r
r@{\hspace{0.2cm}}r@{\hspace{0.2cm}}r@{\hspace{0.2cm}}r@{\hspace{0.2cm}}
r@{\hspace{0.2cm}}r@{\hspace{0.2cm}}r@{\hspace{0.2cm}}r}
\toprule
& \multicolumn{3}{c}{${\nrankT} \leq $} & & \multicolumn{4}{c}{${\nrankF} \leq$} \\
\cmidrule{2-4} \cmidrule{6-9}
Data & {\nrankI} & {\nrankC} & {\nrankA} & & {\nrankI} & {\nrankC} & {\nrankA} & {\nrankT}\\
\midrule
\dtname{gen-ind} & $.84$ & $.97$ & $.78$ &  & $1$ & $.99$ & $.81$ & $.95$ \\
\dtname{gen-copy} & $.79$ & $.96$ & $.82$ &  & $1$ & $.99$ & $.86$ & $.99$ \\
\dtname{Accidents} & $1$ & $.16$ & $.13$ &  & $1$ & $.94$ & $.66$ & $.95$ \\
\dtname{Kosarak} & $1$ & $.08$ & $.08$ &  & $1$ & $.38$ & $.36$ & $.96$ \\
\dtname{Paleo} & $.99$ & $.47$ & $.58$ &  & $1$ & $.91$ & $.85$ & $.86$ \\
\dtname{POS} & $1$ & $.12$ & $.12$ &  & $1$ & $.65$ & $.62$ & $.90$ \\
\dtname{Retail} & $.92$ & $.77$ & $.80$ &  & $1$ & $.94$ & $.92$ & $.62$ \\
\dtname{WebView-1} & $1$ & $.19$ & $.19$ &  & $1$ & $.52$ & $.49$ & $.93$ \\
\dtname{WebView-2} & $.99$ & $.37$ & $.46$ &  & $1$ & $.89$ & $.87$ & $.79$ \\
\bottomrule
\end{tabular}
\caption{Percentages of queries in which the flexible measures {\nrankT} and
{\nrankF} outperform the other rank measures. Queries only of size $3$ or larger
were considered.}
\label{tab:perf_flex}
\end{table}

We studied the sizes of itemsets occurring in $\ifam{F}^*$, the family of known
itemsets in {\nrankF}. To be more precise, let $\ifam{F}^*_G$ be the family of
known itemsets for the query $G$. Let $L$ be the size of itemsets we are
interested in. We define the ratio $r_L$ to be
\[
r_L = \frac{\sum_{G} \abs{\set{X \in \ifam{F}^*_G; \abs{X} = L}}}{\sum_G {\abs{G} \choose L}},
\]
that is, the number of itemset of size $L$ occurring in $\ifam{F}^*$ divided by
the maximum number of occurrences. The ratios $r_L$ are given in
Table~\ref{tab:best_used}. We see that the itemsets of size $2$ and $3$ are
frequently used, however, the itemsets of larger size are rarely used.

\begin{table}
\center
\begin{tabular}{r rrrr}
\toprule
& \multicolumn{4}{c}{Ratio of used itemsets} \\
\cmidrule{2-5}
Data & 2 & 3 & 4 & 5 \\
\midrule
\dtname{gen-ind} & $.35$ & $.01$ & $0$ & $0$ \\
\dtname{gen-copy} & $.36$ & $.01$ & -- & -- \\
\dtname{Accidents} & $.87$ & $.36$ & $.02$ & $0$ \\
\dtname{Kosarak} & $.95$ & $.49$ & $.01$ & -- \\
\dtname{Paleo} & $.74$ & $.16$ & $0$ & -- \\
\dtname{POS} & $.96$ & $.47$ & $.01$ & $0$ \\
\dtname{Retail} & $.62$ & $.13$ & $0$ & $0$ \\
\dtname{WebView-1} & $.93$ & $.44$ & $0$ & -- \\
\dtname{WebView-2} & $.81$ & $.26$ & $.07$ & $0$ \\
\bottomrule
\end{tabular}
\caption{Number of itemsets occurring in $\ifam{F}^*$, the family of known
itemsets in {\rankF}, normalised by the maximum number of possible occurrences.
Each column represent itemsets of specific size.}
\label{tab:best_used}
\end{table}

\subsection{Rank vs.~Other Methods}
We compared our measures against the other ranking methods described in
Section~\ref{sec:setup}.  Namely, we calculated the correlations of {\nrankI},
{\nrankC}, {\nrankA}, {\nrankT}, and {\nrankF} against the frequency of $G$,
$\brin{G}$, the $\chi^2$ test for independency, and $\colstr{G}$, the
collective strength of the itemset $G$. The results are presented in
Tables~\ref{tab:corrs2}~and~\ref{tab:corrs3}.  We also studied the
relationships by plotting our measures as functions of the aforementioned
approaches and such examples are given in Figure~\ref{fig:rank_vs_other}.

\begin{table*}
\center
\begin{tabular}{r
r@{\hspace{0.2cm}}r@{\hspace{0.2cm}}r@{\hspace{0.5cm}}
r@{\hspace{0.2cm}}r@{\hspace{0.2cm}}r@{\hspace{0.5cm}}
r@{\hspace{0.2cm}}r@{\hspace{0.2cm}}r}
\toprule
& \multicolumn{3}{c}{{\nrankI} vs.}
& \multicolumn{3}{c}{{\nrankC} vs.}
& \multicolumn{3}{c}{{\nrankA} vs.} \\
\cmidrule{2-4}
\cmidrule{5-7}
\cmidrule{8-10}
Data & freq. & $\brin{G}$ & $\colstr{G}$ & freq. & $\brin{G}$ & $\colstr{G}$ & freq. & $\brin{G}$ & $\colstr{G}$ \\
\midrule
\dtname{gen-ind} & $.06$ & $.99$ & $-.01$ & $.03$ & $.72$ & $-.01$ & $0$ & $.25$ & $0$ \\
\dtname{gen-copy} & $.15$ & $1$ & $.02$ & $.07$ & $.52$ & $.02$ & $.01$ & $.27$ & $.01$ \\
\dtname{Accidents} & $.01$ & $1$ & $.02$ & $-.01$ & $.17$ & $.05$ & $.03$ & $.07$ & $.01$ \\
\dtname{Kosarak} & $.01$ & $.98$ & $.20$ & $.01$ & $.14$ & $.27$ & $0$ & $.13$ & $.21$ \\
\dtname{Paleo} & $.18$ & $.95$ & $.39$ & $.01$ & $.15$ & $.10$ & $-.03$ & $.20$ & $.03$ \\
\dtname{POS} & $.05$ & $.99$ & $.22$ & $.09$ & $.12$ & $.20$ & $.07$ & $.10$ & $0$ \\
\dtname{Retail} & $.04$ & $.97$ & $.31$ & $.05$ & $.57$ & $.17$ & $.05$ & $.61$ & $.25$ \\
\dtname{WebView-1} & $.06$ & $.98$ & $.19$ & $.07$ & $.15$ & $-.29$ & $.05$ & $.13$ & $-.32$ \\
\dtname{WebView-2} & $.12$ & $.96$ & $.33$ & $.17$ & $.36$ & $.39$ & $.12$ & $.43$ & $.25$ \\
\bottomrule
\end{tabular}
\caption{Correlations between the rank measures {\nrankI}, {\nrankC}, and
{\nrankA} and the base measures: the frequency of $G$, $\brin{G}$,
the $\chi^2$ test for independency, and $\colstr{G}$, the collective strength
of the itemset $G$.}
\label{tab:corrs2}
\end{table*}

\begin{table*}
\center
\begin{tabular}{r
r@{\hspace{0.2cm}}r@{\hspace{0.2cm}}r@{\hspace{0.5cm}}
r@{\hspace{0.2cm}}r@{\hspace{0.2cm}}r}
\toprule
& \multicolumn{3}{c}{{\nrankT} vs.}
& \multicolumn{3}{c}{{\nrankF} vs.} \\
\cmidrule{2-4}
\cmidrule{5-7}
Data & freq. & $\brin{G}$ & $\colstr{G}$ & freq. & $\brin{G}$ & $\colstr{G}$ \\
\midrule
\dtname{gen-ind} & $.07$ & $.79$ & $-.01$ & $.06$ & $.79$ & $-.01$ \\
\dtname{gen-copy} & $.16$ & $.62$ & $.03$ & $.16$ & $.64$ & $.03$ \\
\dtname{Accidents} & $-.02$ & $.33$ & $.04$ & $.04$ & $.07$ & $.01$ \\
\dtname{Kosarak} & $.01$ & $.39$ & $.32$ & $0$ & $.13$ & $.25$ \\
\dtname{Paleo} & $.24$ & $.28$ & $.44$ & $.09$ & $.12$ & $.03$ \\
\dtname{POS} & $.12$ & $.41$ & $.37$ & $.07$ & $.05$ & $-.23$ \\
\dtname{Retail} & $.06$ & $.58$ & $.31$ & $.06$ & $.53$ & $.14$ \\
\dtname{WebView-1} & $.11$ & $.55$ & $.16$ & $.04$ & $.13$ & $-.35$ \\
\dtname{WebView-2} & $.20$ & $.49$ & $.46$ & $.16$ & $.28$ & $.14$ \\
\bottomrule
\end{tabular}
\caption{Correlations between the rank measures {\nrankT} and {\nrankF}
and the base measures: the frequency of $G$, $\brin{G}$,
the $\chi^2$ test for independency, and $\colstr{G}$, the collective strength
of the itemset $G$.}
\label{tab:corrs3}
\end{table*}

\begin{figure*}
\centering
\includegraphics[width=0.25\textwidth]{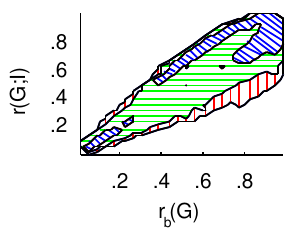}%
\includegraphics[width=0.25\textwidth]{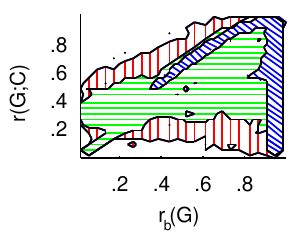}%
\includegraphics[width=0.25\textwidth]{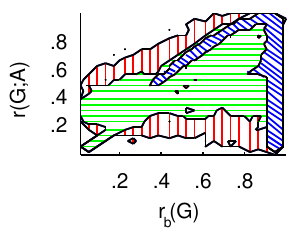}%
\includegraphics[width=0.25\textwidth]{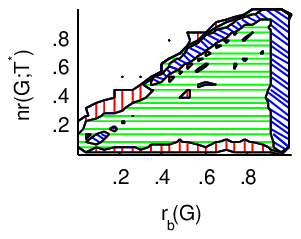}
\includegraphics[width=0.25\textwidth]{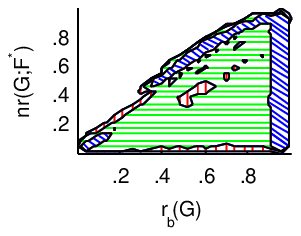}%
\includegraphics[width=0.25\textwidth]{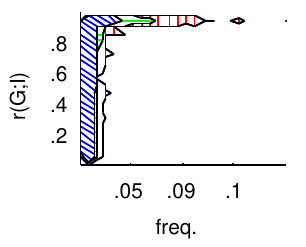}%
\includegraphics[width=0.25\textwidth]{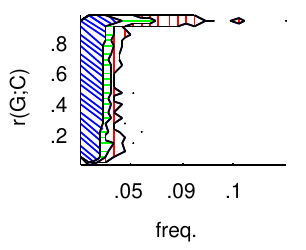}%
\includegraphics[width=0.25\textwidth]{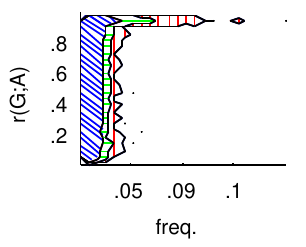}
\includegraphics[width=0.25\textwidth]{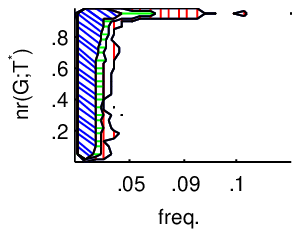}%
\includegraphics[width=0.25\textwidth]{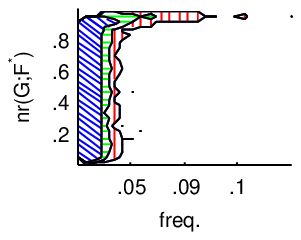}%
\includegraphics[width=0.25\textwidth]{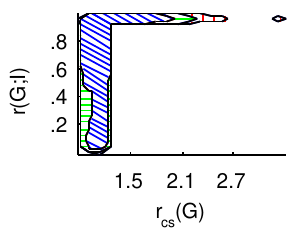}%
\includegraphics[width=0.25\textwidth]{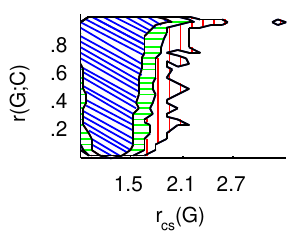}
\includegraphics[width=0.25\textwidth]{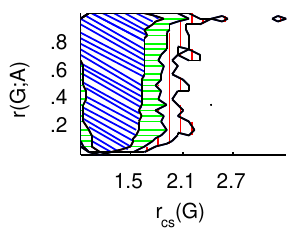}%
\includegraphics[width=0.25\textwidth]{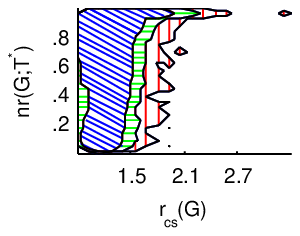}%
\includegraphics[width=0.25\textwidth]{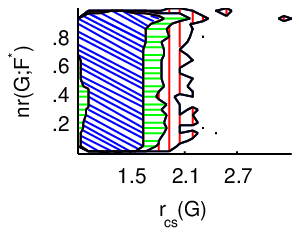}%
\includegraphics[width=0.25\textwidth]{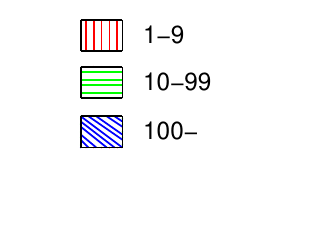}
\caption{Ranks as functions of the base measures. The plots are calculated
from \dtname{Paleo} dataset.}
\label{fig:rank_vs_other}
\end{figure*}

Our first observation is that {\nrankI} correlates strongly with $\brin{G}$.
This is an expected result since both test the independency of attributes
inside the itemsets and also because {\nrankI} is asymptotically a $\chi^2$
test (see~Theorem~\ref{thr:x2generic}).  There is some correlation between
$\brin{G}$ and the rest of the measures although this correlation is much
weaker compared to {\nrankI}.

Apart from \dtname{WebView-2}, there is little correlation between the measures
and the frequency.

The correlation between the measures and the collective strength $\colstr{G}$
exists but varies depending on the method and the dataset. The strongest
correlations are obtained when $\colstr{G}$ is compared against {\nrankI} or
{\nrankT}.  The dependency between {\nrankI} and $\colstr{G}$ is a natural
result since $\colstr{G}$ produces small values when attributes are
independent.

\subsection{Monotonicity of Rank}
In this section we investigate the relationship between the rank of an itemset
and the ranks of its sub-itemsets. Namely, we tested whether the measures are
monotonic, that is, whether $\nrank[\ifam{F}]{G} \geq \nrank[\ifam{F}]{H}$ for
all $H \subset G$. We deliberately ignored sub-itemsets having size $1$ since
they all have very high rank. We also tested whether the measures are
anti-monotonic, that is, decreasing w.r.t.\ set inclusion.

\begin{table*}
\center
\begin{tabular}{r
r@{\hspace{0.2cm}}r@{\hspace{0.2cm}}r@{\hspace{0.2cm}}r@{\hspace{0.2cm}}r
r@{\hspace{0.2cm}}r@{\hspace{0.2cm}}r@{\hspace{0.2cm}}r@{\hspace{0.2cm}}r
r@{\hspace{0.2cm}}r@{\hspace{0.2cm}}r@{\hspace{0.2cm}}r@{\hspace{0.2cm}}r}
\toprule
& \multicolumn{5}{c}{\nrankI} & \multicolumn{5}{c}{\nrankC} & \multicolumn{5}{c}{\nrankA} \\
\cmidrule{2-16}
Data & 3 & 4 & 5 & 6 & All & 3 & 4 & 5 & 6 & All &  3 & 4 & 5 & 6 & All\\
\midrule
\dtname{gen-ind} & $.09$ & $.02$ & $.01$ & $0$ & $.03$ & $.27$ & $.03$ & $.01$ & $0$ & $.05$ & $.27$ & $.11$ & $.06$ & $.03$ & $.11$ \\
\dtname{gen-copy} & $.15$ & $.01$ & -- & -- & $.04$ & $.20$ & $.02$ & -- & -- & $.05$ & $.20$ & $.10$ & -- & -- & $.12$ \\
\dtname{Accidents} & $.78$ & $.92$ & $.97$ & $.99$ & $.90$ & $.01$ & $.02$ & $.02$ & $0$ & $.02$ & $.01$ & $0$ & $0$ & $0$ & $0$ \\
\dtname{Kosarak} & $.93$ & $.98$ & $1$ & -- & $.93$ & $0$ & $0$ & $0$ & -- & $0$ & $0$ & $0$ & $0$ & -- & $0$ \\
\dtname{Paleo} & $.40$ & $.61$ & $.84$ & -- & $.51$ & $.04$ & $0$ & $0$ & -- & $.02$ & $.04$ & $0$ & $0$ & -- & $.02$ \\
\dtname{POS} & $.87$ & $1$ & $1$ & $1$ & $.92$ & $0$ & $0$ & $.01$ & $0$ & $0$ & $0$ & $0$ & $0$ & $0$ & $0$ \\
\dtname{Retail} & $.11$ & $.42$ & $.92$ & $1$ & $.19$ & $.04$ & $0$ & $0$ & $0$ & $.03$ & $.04$ & $.02$ & $0$ & $0$ & $.04$ \\
\dtname{WebView-1} & $.98$ & $1$ & $1$ & -- & $.98$ & $.04$ & $0$ & $0$ & -- & $.04$ & $.04$ & $0$ & $0$ & -- & $.04$ \\
\dtname{WebView-2} & $.39$ & $.88$ & $1$ & $1$ & $.67$ & $.04$ & $0$ & $.08$ & $1$ & $.02$ & $.04$ & $0$ & $0$ & $0$ & $.02$ \\
\bottomrule
\end{tabular}
\caption{Percentages of itemsets satisfying the property of monotonicity. The
itemset $G$ satisfies the property if $\nrank[\ifam{F}]{G} \geq
\nrank[\ifam{F}]{H}$ for all $H \subset G$ such that $\abs{H} \geq 2$.}
\label{tab:mono}
\end{table*}

\begin{table*}
\center
\begin{tabular}{r
r@{\hspace{0.2cm}}r@{\hspace{0.2cm}}r@{\hspace{0.2cm}}r@{\hspace{0.2cm}}r
r@{\hspace{0.2cm}}r@{\hspace{0.2cm}}r@{\hspace{0.2cm}}r@{\hspace{0.2cm}}r}
\toprule
& \multicolumn{5}{c}{\nrankT} & \multicolumn{5}{c}{\nrankF} \\
\cmidrule{2-11}
Data & 3 & 4 & 5 & 6 & All & 3 & 4 & 5 & 6 & All \\
\midrule
\dtname{gen-ind} & $.13$ & $.01$ & $0$ & $0$ & $.02$ & $.07$ & $.01$ & $0$ & $0$ & $.02$ \\
\dtname{gen-copy} & $.10$ & $.01$ & -- & -- & $.02$ & $.05$ & $.01$ & -- & -- & $.02$ \\
\dtname{Accidents} & $.02$ & $.13$ & $.35$ & $.47$ & $.17$ & $.01$ & $0$ & $0$ & $0$ & $0$ \\
\dtname{Kosarak} & $0$ & $.26$ & $.32$ & -- & $.03$ & $0$ & $0$ & $0$ & -- & $0$ \\
\dtname{Paleo} & $.02$ & $0$ & $0$ & -- & $.01$ & $.02$ & $0$ & $0$ & -- & $.01$ \\
\dtname{POS} & $0$ & $.23$ & $.97$ & $1$ & $.11$ & $0$ & $0$ & $0$ & $0$ & $0$ \\
\dtname{Retail} & $.03$ & $0$ & $.14$ & $1$ & $.02$ & $.02$ & $0$ & $0$ & $0$ & $.02$ \\
\dtname{WebView-1} & $.04$ & $.07$ & $.89$ & -- & $.05$ & $.03$ & $0$ & $0$ & -- & $.03$ \\
\dtname{WebView-2} & $.03$ & $.02$ & $.72$ & $1$ & $.04$ & $.03$ & $0$ & $0$ & $0$ & $.01$ \\
\bottomrule
\end{tabular}
\caption{Percentages of itemsets satisfying the property of monotonicity. The
itemset $G$ satisfies the property if $\nrank[\ifam{F}]{G} \geq
\nrank[\ifam{F}]{H}$ for all $H \subset G$ such that $\abs{H} \geq 2$.}
\label{tab:mono2}
\end{table*}

\begin{table*}
\center
\begin{tabular}{r
r@{\hspace{0.2cm}}r@{\hspace{0.2cm}}r@{\hspace{0.2cm}}r@{\hspace{0.2cm}}r
r@{\hspace{0.2cm}}r@{\hspace{0.2cm}}r@{\hspace{0.2cm}}r@{\hspace{0.2cm}}r
r@{\hspace{0.2cm}}r@{\hspace{0.2cm}}r@{\hspace{0.2cm}}r@{\hspace{0.2cm}}r}
\toprule
& \multicolumn{5}{c}{\nrankI} & \multicolumn{5}{c}{\nrankC} & \multicolumn{5}{c}{\nrankA} \\
\cmidrule{2-16}
Data & 3 & 4 & 5 & 6 & All & 3 & 4 & 5 & 6 & All & 3 & 4 & 5 & 6 & All\\
\midrule
\dtname{gen-ind} & $.21$ & $.07$ & $.03$ & $.02$ & $.07$ & $.25$ & $.07$ & $.03$ & $.01$ & $.07$ & $.25$ & $.08$ & $.02$ & $.01$ & $.07$ \\
\dtname{gen-copy} & $.15$ & $.06$ & -- & -- & $.08$ & $.25$ & $.08$ & -- & -- & $.11$ & $.25$ & $.07$ & -- & -- & $.10$ \\
\dtname{Accidents} & $.03$ & $0$ & $0$ & $0$ & $.01$ & $.62$ & $.04$ & $0$ & $0$ & $.16$ & $.62$ & $.21$ & $.05$ & $.02$ & $.26$ \\
\dtname{Kosarak} & $.02$ & $.06$ & $0$ & -- & $.02$ & $.93$ & $.03$ & $.01$ & -- & $.83$ & $.93$ & $.33$ & $.08$ & -- & $.86$ \\
\dtname{Paleo} & $.02$ & $0$ & $0$ & -- & $.01$ & $.43$ & $.04$ & $0$ & -- & $.22$ & $.43$ & $.07$ & $0$ & -- & $.23$ \\
\dtname{POS} & $.01$ & $.01$ & $.04$ & $.23$ & $.01$ & $.87$ & $.07$ & $0$ & $0$ & $.57$ & $.87$ & $.18$ & $.06$ & $.09$ & $.61$ \\
\dtname{Retail} & $.17$ & $0$ & $0$ & $0$ & $.13$ & $.38$ & $.05$ & $.01$ & $0$ & $.30$ & $.38$ & $.03$ & $.01$ & $0$ & $.29$ \\
\dtname{WebView-1} & $0$ & $0$ & $0$ & -- & $0$ & $.69$ & $.11$ & $0$ & -- & $.62$ & $.69$ & $.39$ & $.15$ & -- & $.66$ \\
\dtname{WebView-2} & $.07$ & $.01$ & $.14$ & $.96$ & $.04$ & $.48$ & $.06$ & $0$ & $0$ & $.24$ & $.48$ & $.06$ & $.07$ & $.04$ & $.25$ \\
\bottomrule
\end{tabular}
\caption{Percentages of itemsets satisfying the property of anti-monotonicity.
The itemset $G$ satisfies the property if $\nrank[\ifam{F}]{G} \leq
\nrank[\ifam{F}]{H}$ for all $H \subset G$ such that $\abs{H} \geq 2$.}
\label{tab:amono}
\end{table*}

\begin{table*}
\center
\begin{tabular}{r
r@{\hspace{0.2cm}}r@{\hspace{0.2cm}}r@{\hspace{0.2cm}}r@{\hspace{0.2cm}}r
r@{\hspace{0.2cm}}r@{\hspace{0.2cm}}r@{\hspace{0.2cm}}r@{\hspace{0.2cm}}r}
\toprule
& \multicolumn{5}{c}{\nrankT} & \multicolumn{5}{c}{\nrankF} \\
\cmidrule{2-11}
Data & 3 & 4 & 5 & 6 & All & 3 & 4 & 5 & 6 & All \\
\midrule
\dtname{gen-ind} & $.45$ & $.16$ & $.06$ & $.02$ & $.15$ & $.53$ & $.16$ & $.05$ & $.01$ & $.15$ \\
\dtname{gen-copy} & $.43$ & $.18$ & -- & -- & $.22$ & $.51$ & $.17$ & -- & -- & $.23$ \\
\dtname{Accidents} & $.58$ & $.01$ & $0$ & $0$ & $.13$ & $.69$ & $.29$ & $.07$ & $.02$ & $.32$ \\
\dtname{Kosarak} & $.91$ & $0$ & $0$ & -- & $.81$ & $.95$ & $.50$ & $.07$ & -- & $.90$ \\
\dtname{Paleo} & $.52$ & $.02$ & $0$ & -- & $.25$ & $.56$ & $.15$ & $.01$ & -- & $.34$ \\
\dtname{POS} & $.88$ & $0$ & $0$ & $0$ & $.56$ & $.90$ & $.47$ & $.13$ & $0$ & $.73$ \\
\dtname{Retail} & $.72$ & $.02$ & $0$ & $0$ & $.55$ & $.75$ & $.13$ & $.03$ & $0$ & $.60$ \\
\dtname{WebView-1} & $.62$ & $0$ & $0$ & -- & $.55$ & $.70$ & $.56$ & $.22$ & -- & $.68$ \\
\dtname{WebView-2} & $.69$ & $.01$ & $0$ & $0$ & $.31$ & $.71$ & $.23$ & $.06$ & $0$ & $.44$ \\
\bottomrule
\end{tabular}
\caption{Percentages of itemsets satisfying the property of anti-monotonicity.
The itemset $G$ satisfies the property if $\nrank[\ifam{F}]{G} \leq
\nrank[\ifam{F}]{H}$ for all $H \subset G$ such that $\abs{H} \geq 2$.}
\label{tab:amono2}
\end{table*}

From the results given in Tables~\ref{tab:mono}--\ref{tab:amono2} our first
observation is that {\nrankI} are increasing for real datasets but not for the
synthetic datasets. The raw values of {\nrankI} are indeed increasing but this
does not hold for the P-values since the number of degrees varies.  The measure
{\nrankT} tends also be monotonic but not as much as {\nrankI}.  On the
contrary, {\nrankC}, {\nrankA}, and {\nrankF} are increasing for extremely few
itemsets.

Table~\ref{tab:amono} suggests that {\nrankC}, {\nrankA}, and {\nrankF}
satisfies the anti-monotonicity to some degree. Measures {\nrankC} and
{\nrankA} are anti-monotonic for relatively high percentage of itemsets of size
$3$. Among itemsets of size $4$, {\nrankF} satisfies the property of
anti-monotonicity for a slightly larger portion of itemsets than {\nrankA}
that, in turn, is anti-monotonic in more queries than {\nrankC}.

\section{Conclusions}
\label{sec:conclusions}
We have given a definition of a measure for ranking itemsets. The idea is to
predict the frequency of an itemset from the frequencies of its sub-itemsets
and measure the deviation between the actual frequency and the prediction. The
more the itemset deviates from the prediction, the more it is significant. We
estimated the frequencies using Maximum Entropy and we used Kullback-Leibler
divergence to measure the deviation. In the general case, the measure can be
computed in $O(2^{\abs{G}})$ time, where $\abs{G}$ is the size of the itemset
needed to be ranked, however, the measures {\rankT} and {\rankI} can be
computed in polynomial time.

We introduced two flexible rank measures {\rankT} and {\rankF}. The measure
{\rankT} can be solved by finding the optimal spanning tree in the mutual
information matrix. For solving {\rankF} we proposed a simple greedy approach.

A clear advantage of our approach to the previous methods is that the previous
solutions calculate the deviation from the independence model whereas we are
able to use the information available from the itemsets of larger size, and
thus use more flexible models.

Our empirical results for real data show that the independence is too strict
assumption: Almost all itemsets were significant according to {\rankI}. The
results changed when we applied the more flexible models, {\rankC} and
{\rankA}. We also observed an interesting type of overfitting: In some cases we
obtain a better prediction if we do not use all the available information.

We showed that there is a little correlation between our measures and the other
approaches.  For instance, infrequent itemset may be significant and frequent
itemset may be insignificant. We also observed that {\rankI} is monotonic for a
large portion of itemsets, whereas {\rankC} and {\rankA} are anti-monotonic for
a significant portion of itemsets.

\section*{Acknowledgments}
The author wishes to thank Gemma Garriga, Heikki Mannila, and Robert Gwadera
for their comments.
\bibliography{rank}
\appendix
\section{Asymptotic Behaviour of the Divergence}
\label{sec:assymptotic}
By asymptotic behaviour we mean the following: We assume that
we have an ensemble of datasets $D_i$ such that $\abs{D_i} \to \infty$. We
assume that $G$ is non-derivable in each $D_i$ and that the frequencies of
$\ifam{F}_G$ are all equal.

Define $N = \abs{D}$ and $M = \abs{\ifam{H}}$. Let $\mathbb{P}$ be the set of
distributions satisfying the itemsets $\ifam{F}_G$.  It is easy to see that we
can parameterize $\mathbb{P}$ with frequencies of $\ifam{H}$.  In other words,
let $\ifam{H} = \enset{H_1}{H_M}$.  Then for each $p \in \mathbb{P}$, there is
a unique frequency vector $\theta \in \real^M$ such that $\theta_i = p(H_i =
1)$. Let $\Theta$ be the set of all possible frequency vectors. The set
$\Theta$ is a closed polytope --- the vectors located on the boundary of
$\Theta$ corresponds to the distributions in which at least one entry is $0$.

Let $\theta^{*}$ be a frequency vector corresponding to the Maximum Entropy
distribution $\pme$.  We need to show that $\theta^{*}$ is not a boundary
vector. Assume the converse, then $\pme$ must have $\pme(\omega) = 0$ for
some $\omega$. We know that this implies that $p(\omega) = 0$ for all $p \in
\mathbb{P}$~\cite[Theorem 3.1]{csiszar75divergence}. Let $Y$ be the itemset
containing the elements for which $\omega$ has positive entries. This in turns
(see~\cite{calders02mining}) implies that for each $p \in \mathbb{P}$
\[
p(G = 1) = \sum_{Y \subseteq Z \subseteq G} (-1)^{\abs{G} - \abs{Z}} p(Z = 1),
\]
making $G$ derivable and contradicting the statement.

Since $\theta^{*}$ is an inner point of $\Theta$, let $B \subset \Theta$ be an
open ball around $\theta^{*}$. Assume that $\theta \in B$. By taking the expectation
of the second-degree Taylor expansion of $\log\frac{p(\omega ; \theta^{*})}{p(\omega; \theta)}$ around $\theta$
we arrive to
\[
	-\kl{\theta}{\theta^{*}} = \frac{1}{2} \Delta\theta^T\mean{\theta}{H(\omega; \eta)} \Delta\theta,
\]
where $\Delta\theta = \theta^{*} - \theta$ and $\eta$ is a vector lying between
$\theta$ and $\theta^{*}$, and $H$ is the Hessian matrix of $\log p(\omega; \eta)$. 

Let $\theta_N$ be the frequencies of $\ifam{H}$ obtained from a dataset
containing $N$ points. According to $0$-hypothesis we have $\theta_N \leadsto
\theta^{*}$ and $\sqrt{N}\pr{\theta_N - \theta^{*}} \leadsto N(0, \Sigma)$,
where $\Sigma$ is a covariance matrix,
\[
	\Sigma_{ij} = \pme(H_i = 1, H_j = 1) - \pme(H_i = 1)\pme(H_j = 1). 
\]
If $\theta_N \in B$, we let $\eta_N$ correspond to $\eta$ in the Taylor
expansion, otherwise we set $\eta_N = 0$. We can show that $\eta_N \leadsto
\theta^{*}$~\cite[Theorem 2.7]{vaart98statistics}. Consider a function
\[
g(a, b, c, d) =
\begin{cases}
-a^T \mean{c}{H(\omega; b)} a & c \in B \\
\pr{2/d}\kl{c}{\theta^{*}} & c \notin B
\end{cases}.
\]
This function is continuous in $\pr{\real^M, \theta^{*}, \theta^{*}, 0}$.
Hence, we can apply continuous map theory~\cite[Theorem 2.3]{vaart98statistics}
to obtain that
\[
	2N\kl{\theta_N}{\theta^{*}} = g\pr{\sqrt{N}\pr{\theta_N - \theta^{*}}, \eta_N, \theta_N, \frac{1}{N}}
	\leadsto -X^T\mean{\theta^{*}}{H\pr{\omega; \theta^{*}}}X,
\]
where $X$ is a random variable distributed as $N(0, \Sigma)$. We know that
$\mean{\theta^{*}}{H\pr{\omega; \theta^{*}}} = -\Sigma^{-1}$~\cite[Lemma
4.11]{kullback68information}. Theorem follows since $X^T\Sigma^{-1}X$ is
distributed as $\chi^2$ with $M$ degrees of freedom~\cite[Lemma
17.1]{vaart98statistics}.

\end{document}